\documentclass[twoside]{article}
\usepackage[accepted]{aistats2017}

\usepackage[utf8]{inputenc} % allow utf-8 input
\usepackage[T1]{fontenc}    % use 8-bit T1 fonts
\usepackage{url}            % simple URL typesetting
\usepackage{booktabs}       % professional-quality tables
\usepackage{amsfonts}       % blackboard math symbols
\usepackage{nicefrac}       % compact symbols for 1/2, etc.
\usepackage[colorlinks=true, allcolors=blue]{hyperref}
\usepackage{amsmath}
\usepackage{amssymb}
\usepackage{amsthm}
\usepackage{graphicx}
\usepackage{subfig}
\usepackage{caption}
\usepackage{epstopdf}
\usepackage{algorithm}
\usepackage[noend]{algorithmic}
\usepackage{natbib}

\graphicspath{{figures/}}

\newcommand{\expect}[1]{\mathbb{E}\left[#1\right]}
\newcommand{\var}[1]{\text{Var}\left(#1\right)}

\newcommand{\pr}[1]{\Pr\left[#1\right]}
\newcommand{\x}{{\bf x}}

\newcommand{\simiid}{\overset{i.i.d.}{\sim}}

\newcommand{\M}{{\mathcal M}}
\DeclareMathOperator*{\argmin}{\arg\!\min}

\newtheorem{lemma}{Lemma}
\newtheorem{theorem}{Theorem}

\theoremstyle{definition}

\theoremstyle{remark}
\newtheorem{remark}{Remark}
\begin{document}

\runningauthor{Amit Moscovich, Ariel Jaffe, Boaz Nadler}
% If your paper is accepted and the title of your paper is very long,
% the style will print as headings an error message. Use the following
% command to supply a shorter title of your paper so that it can be
% used as headings.
%

% If your paper is accepted and the number of authors is large, the
% style will print as headings an error message. Use the following
% command to supply a shorter version of the authors names so that
% they can be used as headings (for example, use only the surnames)
%
%\runningauthor{Surname 1, Surname 2, Surname 3, ...., Surname n}

\twocolumn[
        \aistatstitle{Minimax-optimal semi-supervised regression on unknown manifolds}
        \aistatsauthor{Amit Moscovich \And  Ariel Jaffe \And  Boaz Nadler}
        \aistatsaddress{Weizmann Institute of Science \And Weizmann Institute of Science \And Weizmann Institute of Science}
]

\begin{abstract}
    We consider semi-supervised regression when the predictor variables are drawn from an unknown manifold.
    A simple two step approach to this problem is to:
    (i) estimate the manifold geodesic distance between any pair of points using both the labeled and unlabeled instances; 
    and (ii) apply a $k$ nearest neighbor regressor based on these distance estimates.
    We prove that given sufficiently many unlabeled points, this simple method of \emph{geodesic kNN regression}
    achieves the optimal \emph{finite-sample} minimax bound on the
    mean squared error, as if the manifold were known.
    Furthermore, we show how this approach can be efficiently implemented,
    requiring only $O(k N \log N)$ operations to estimate 
    the regression function at all \(N \) labeled and unlabeled points.
    We illustrate this approach on two datasets with a manifold structure:
    indoor localization using WiFi fingerprints and
    facial pose estimation.
    In both cases, geodesic kNN is more accurate and much faster than the
    popular Laplacian eigenvector regressor.
\end{abstract}

\section{Introduction}
In recent years, many semi-supervised regression and classification methods have been proposed,
see the surveys by \cite{Chapelle2006, zhu2009introduction, SubramanyaTalukdar2014}.
These methods demonstrated empirical success on some data sets,
whereas on others the unlabeled data did not appear to help.
This raised two key questions of continued interest:
(i) under which conditions can the potentially huge amount of unlabeled data help the learning process?
and (ii) can we design statistically sound and computationally efficient methods that benefit from the unlabeled data?

The {\em cluster} assumption and the {\em manifold} assumption are two common
models for studying the above questions regarding semi-supervised learning. 
Under the cluster assumption, instances with the same label concentrate in well-defined clusters separated by low density regions \citep{Chapelle2005,Rigollet2007,Singh2009}.
Under the manifold assumption the data points reside in one or several low-dimensional manifolds, with nearby instances on the manifold having similar response values.
%In this case, the learning problem has a lower intrinsic dimension than that of the ambient space.
%When many unlabeled data points are available, they may be used to learn the manifold embedding,
%thus reducing the dimension of the learning problem to the intrinsic dimension of the data manifold.

\cite{BickelLi2007} as well as \cite{LaffertyWasserman2007} studied semi-supervised learning under the manifold assumption.
They showed that without knowing the manifold, standard multivariate polynomial regression in the ambient space,
using only the labeled data, achieves the asymptotic minimax rate for Sobolev functions.
According to these results, it seems there is little benefit to the availability of additional unlabeled data. However, these results require that the number of labeled samples tends to infinity. Intuitively, in this limit, the geometry of the data manifold
and the sampling density can be accurately estimated from the labeled data alone.
Thus the benefits of a potentially huge number of unlabeled points when there is little labeled data remained unclear.

One of the goals of this work is to clarify this benefit of unlabeled data for rather general manifolds, via a finite sample analysis, whereby the number of labeled samples is fixed. In this context, \cite{Niyogi2013} showed that unlabeled data can indeed help, by presenting a specially constructed manifold, for which supervised learning is provably more difficult than semi-supervised learning.  
 \cite{GoldbergZhuSinghXuNowak2009}
considered this question under both the manifold and multi-manifold cases.
In particular, in their Section 2.1, they conjectured that semi-supervised learning of a H\"older function on an unknown manifold with intrinsic dimension $d$
can achieve the finite-sample minimax bound for nonparametric regression in $\mathbb{R}^d$.

In this paper we prove that when the regressed function is Lipschitz, a simple semi-supervised regression method
based on geodesic nearest neighbor averaging
achieves the finite-sample minimax bound
when the amount of unlabeled points is sufficiently large.
This settles the conjecture of \cite{GoldbergZhuSinghXuNowak2009}
for the Lipschitz case.

The regression method we consider, denoted \emph{geodesic kNN regression}, consists of two steps:
(i) estimate the manifold geodesic distances
by shortest-path distances in a graph constructed from both the labeled and unlabeled points; and (ii) estimate the response at any point by averaging its $k$ geodesic nearest labeled neighbors.
Section \ref{sec:framework} describes the graph construction and the corresponding nonparametric statistical estimation method.
Our main result, detailed in Section \ref{sec:stat_analysis}, is a proof that for a Lipschitz function on a manifold,
if enough unlabeled samples are available, 
then with high probability this method achieves the finite-sample minimax bound.
In Section \ref{sec:fastgknn} we discuss the computational aspects of this approach,
which is very fast compared to spectral-based semi-supervised methods.
Finally, in Section \ref{sec:application} we apply our method to two problems with a low dimensional manifold structure, indoor localization using WiFi fingerprints
and facial pose estimation.
On both problems geodesic kNN exhibits a marked improvement
compared to classical kNN, which does not utilize the unlabeled data,
and also compared to the popular semi-supervised regression method of \cite{BelkinNiyogi2004}.

\section{Semi-supervised learning with geodesic distances} \label{sec:framework}

We consider the following framework for semi-supervised learning.
Given $n$ labeled instances $\mathcal{L} = \{(\x_i, y_i)\}_{i=1}^n $
and $m$ unlabeled instances $\mathcal U = \{\x_j\}_{j=1}^m$ from an instance space $\mathcal{X}$
equipped with a distance function $d(\x,\x')$:
\begin{enumerate}
    \item Construct an undirected (sparse)\ graph $G$ whose vertices
    are all the labeled and unlabeled points.
    Pairs of close points $\x, \x'$ are then connected by an edge with weight $w(\x, \x') = d(\x, \x')$.
%     \item Construct an undirected (sparse)\ graph $G$ whose vertices
%     are the set of all points $\mathcal L\cup\mathcal U =\{\x_1, \ldots, \x_{n+m}\}$. Pairs of distant points are assigned
%     an edge with a large weight or no edge at all, whereas     close pairs of points are assigned edges with \textit{small} weights,
% by some rule.   
    
    \item Compute the shortest-path graph distance $d_G(\x_i, \x_j)$ for all $\x_i \in \mathcal L$ and
    $\x_j \in \mathcal L \cup\mathcal U$.
    \item Apply standard metric-based supervised learning methods,
    such as kNN or Nadaraya-Watson, using the computed graph distances $d_G$.
\end{enumerate}

This framework generalizes the work of  \citet{BijralRatliffSrebro2011},
which assumed that the samples are vectors in $\mathbb{R}^D$
and  the distance function is 
\(
    \| \x_i - \x_j\|_p^q.
\)
The use of geodesic nearest neighbors for classification was also considered by \cite{BelkinNiyogi2004}.
Specific edge selection rules include the distance-cutoff rule,
whereby two points are connected by an edge if their distance is below a threshold,  
and the symmetric kNN rule, where every point is connected by an edge to its $k$
nearest neighbors and vice versa.
\citep{AlamgirVonluxburg2012, TingHuangJordan2010}

The elegance of this framework is that it \textit{decouples} the unsupervised and supervised
parts of the learning process.
It represents the geometry of the samples by a single metric $d_G$,
thus enabling the application of any supervised learning algorithm based on a metric.
For classification, a natural choice is the $k$ nearest neighbors algorithm.
For regression, one may similarly employ a $k$ nearest neighbor regressor.
For any $\x_i\in\mathcal L\cup\mathcal U$, let kNN$(\x_i) \subseteq \mathcal{L}$ denote the set of $k$ (or less)\ nearest \emph{labeled} neighbors to $\x_i$, as determined by the graph distance $d_G$.
The \emph{geodesic kNN regressor} at \(\x_i\) is 
\begin{equation} \label{eq:gknn_regressor}
    \hat{f}(\x_{i}) := \frac{1}{|\text{kNN}(\x_i)|}\sum_{(\x_j, y_j) \in \text{kNN}(\x_i)} y_j.
\end{equation}

We now extend the definition of $\hat{f}(\x)$ to the inductive setting.
Assume we have already computed the regression estimates $\hat{f}(\x_i)$ of Eq. \eqref{eq:gknn_regressor}
for all points in $\mathcal{L} \cup \mathcal{U}$. For a new instance $\x \notin \mathcal{L} \cup
\mathcal{U}$,
we first find its \emph{Euclidean} nearest neighbor $\x^*$ from $\mathcal{L} \cup\mathcal{U}$. This can be done in sublinear time either using data structures for spatial queries \citep{Omohundro1989,Bentley1975}
or by employing approximate nearest neighbor methods \citep{AndoniIndyk2006}. 
Then the geodesic kNN regression estimate at $\x$ is
\begin{align}
    \hat{f}(\x) := \hat{f}(\x^*) = \hat{f} \left( \argmin_{\x' \in \mathcal{L} \cup \mathcal{U}} \| \x - \x'\| \right). 
                \label{eq:inductive}
\end{align}

\section{Statistical analysis under the manifold assumption}
\label{sec:stat_analysis}
We now analyze the statistical properties of the geodesic kNN\ regressor \(\hat f\) of Eq. (\ref{eq:inductive}), under the manifold assumption.
We consider a standard nonparametric regression model, $Y = f(X) + \mathcal{N}(0,\sigma^2)$ where
$X \in \mathbb{R}^D$ is drawn according to
a measure $\mu$ on $\mathcal M$.
We prove that if $f$ is Lipschitz with respect to the manifold distance and if enough unlabeled points are available then $\hat{f}$ obtains the minimax bound on the mean squared error.

To this end, we first review some classical results in nonparametric estimation.
Let $\hat f:\mathbb{R}^D \to \mathbb{R}$ be an estimator of a function $f$, based on $n$ noisy samples.
Let 
\(
    \text{MSE}(\hat{f}, \x) := \expect{ ( \hat{f}(\x) - f(\x) )^2}$ be its  mean squared error at a point $\x \in \mathbb{R}^D$,
where the expectation is over the random draw of data points.
It can be shown that for \emph{any} estimator $\hat{f}$ and any point $\x \in \mathbb{R}^D$,
there is some Lipschitz function $f$ such that
\(
    \text{MSE}(\hat{f}, \x) \ge c n^{-\frac{2}{2+D}}
\)
for some constant $c > 0$ that depends only on the Lipschitz constant and the noise level.
The term $n^{-\frac{2}{2+D}}$ is thus termed a \emph{finite-sample minimax lower bound}
on the MSE at a point.
Several results of this type were derived under various measures of risk and
classes of functions \citep{Tsybakov2009,Gyorfi2002}.

Standard nonparametric methods such as Nadaraya-Watson or kNN regression
have an upper bound on their MSE that is also of the form $c' n^{-\frac{2}{2+D}}$.
Hence these methods are termed \emph{minimax optimal}
for estimating a Lipschitz function.

In Theorem \ref{thm:minimax_optimality} below we prove that given a sufficient number of \emph{unlabeled} points,
the MSE of the geodesic kNN regressor is upper-bounded by $cn^{-\frac{2}{2+d}}$
where $c$ is some constant and $d$ is the \textit{intrinsic} dimension of the manifold.
Hence the geodesic kNN regressor is minimax-optimal and adaptive to the geometry of the unknown manifold.

%%%%%%%%%%%%%%%%%%%%%%%%%%%%%%%%%%%%
\subsection{Notation and prerequisites}

Our main result relies on the analysis of \cite{TenenbaumDesilvaLangford2000} regarding the approximation of manifold distances by graph distances.
Before stating our result, we thus first introduce some notation, our assumptions and a description of the key results of \cite{TenenbaumDesilvaLangford2000} that we shall use. 

For a general background on smooth manifolds, see for example the book by \cite{Lee2012}.
We assume that the data manifold $\M \subseteq \mathbb{R}^D$
is a compact smooth manifold of known intrinsic dimension $d$,
possibly with boundaries and corners.
We further assume that $\M$ is geodesically convex,
i.e. that every two points in $\M$ are connected by a geodesic curve.
We denote by $d_{\mathcal M}(\x, \x')$ the length of the shortest path between two
points in $\mathcal M$, the diameter of $\M$ by
\(
    \text{diam}(\M) := \sup_{\x, \x'} d_\M(\x, \x')
\)
and the manifold-ball of points around $\x$ by $B_\x(r) := \{\x' \in \M : d_\M(\x, \x') < r \}$.
We denote the volume of $\M$ by $V$ and the minimum volume of a manifold ball of radius $r$
by $V_{\min}(r) := \min_{\x \in \M} Vol(B_\x(r))$.
We denote by $r_0$ the minimal radius of curvature of $\M$
and by $s_0$ its minimal branch separation (see the supplementary of \cite{TenenbaumDesilvaLangford2000}
for precise definitions).
We assume that the data points are sampled i.i.d. from some measure $\mu$ on $\M$ with associated density function $\mu(\x)$.
For every point $\x \in \M$ and radius $r \le R$ we assume that $\mu(B_{\x}(r)) \ge Q r^d$
where $R,Q > 0$.
This condition means that the measure of small balls grows with the radius as is typical for dimension $d$.
In particular, it guarantees that the minimum density $\mu_{\min}:=\min_{\x\in\mathcal M}\mu(\x)> 0$.
Finally, we assume that $f:\M \to \mathbb{R}$ is a bounded $L$-Lipschitz function on $\M$, 
\begin{align} \label{eq:lipschitz}
    \forall \x, \x' \in \mathcal{M}:|f(\x)-f(\x')| \le L d_{\mathcal{M}}(\x, \x').
\end{align}

We now reproduce the statement of Theorem B.
\paragraph{Theorem B. \citep{TenenbaumDesilvaLangford2000}}
\textit{Let $\M \subseteq \mathbb{R}^D$ be a compact smooth and geodesically convex manifold of intrinsic dimension $d$.
    Let $\delta, \epsilon, r > 0$ be constants.
    Let $X_1, \ldots, X_N \simiid \mu$ be a sample of points on $\M$
    and suppose we use these points to construct a graph $G$ using the distance-cutoff
    rule with threshold $r$ where $r < \min\{s_0, (2/\pi) r_0 \sqrt{24 \delta}\}$.
}

\textit{\noindent Denote by $A$ the event that the following inequalities
%\begin{align} \label{eq:isomap_distance_bounds}                    
%    1-\delta \le \frac{d_G(X_i, X_j)}{d_{\mathcal M}(X_i, X_j)} \le 1+\delta.
%\end{align}
\begin{align} \label{eq:isomap_distance_bounds}                    
    1-\delta \le d_G(X_i, X_j) / d_{\mathcal M}(X_i, X_j) \le 1+\delta.
\end{align}
\textit{hold for all pairs $X_i, X_j$, where $1 \le i,j \le N$}. Then}
\begin{align} \label{eq:pr_A}
    \pr{
        A
        \Big|
        N > \frac{
                \log \left( V / \epsilon V_{\min} \left( \tfrac{\delta r}{16} \right) \right) 
            }{
                \mu_{\min} V_{\min} \left( \tfrac{\delta r}{8} \right)
            }
    }
    \ge
    1 - \epsilon.
\end{align}

\begin{remark}
    By Theorem C in \citep{TenenbaumDesilvaLangford2000}, a similar result holds for the symmetric kNN rule.
\end{remark}
\begin{remark}
    In the typical case where $V_{\min}(r) \sim r^{-d}$,
    if we fix $\epsilon, \delta$ we must have $N \gtrsim \frac{1}{\mu_{\min}} (8 / \delta r)^d$.
    In other words, the required number of samples for Eq. \eqref{eq:isomap_distance_bounds}
    to hold is exponential in the intrinsic dimension $d$.
\end{remark}
\begin{remark}
    If we fix $N, \delta, r$ and invert Eq. \eqref{eq:pr_A}, we conclude that $\pr{A^c}$
    decays exponentially with $N$,
    \begin{align} \label{eq:epsilon_bound}
        \pr{A^c} < \epsilon = c_a e^{-c_b N}
    \end{align}
    where $c_a = V / V_{\min}(\frac{\delta r}{16})$ and $c_b = V_{\min}(\frac{\delta r}{8}) \cdot \mu_{\min}$.
    A similar bound holds for the symmetric kNN graph.
\end{remark}
\begin{remark}
    Theorems B and C consider points drawn from a Poisson point process.
    However, they hold also in the case of an i.i.d. draw of $N$ points.
    See page 11 of the supplement of \cite{TenenbaumDesilvaLangford2000}.
\end{remark}

%    Let $\x \in \M$ be an arbitrary point on the manifold
%    and let $\x^*$ be the closest point to it from either the labeled or unlabeled sample
%    \[
%        \x^* := \argmin_{X_i \in \{ X_1, \ldots, X_{n+m}\}} \| \x - X_i \|.
%    \]
\subsection{Main result}    
We are now ready to state our main theorem.
It bounds the expected MSE of the geodesic kNN regressor $\hat{f}(\x)$ at a fixed point $\x \in \M$,
where the expectation is over the draw of $n$ labeled and $m$ unlabeled points.

\begin{theorem} \label{thm:minimax_optimality}
    Consider a fixed point $\x \in \M$. Suppose the manifold $\M$, the measure $\mu$
    and the regression function $f$ satisfy all the assumptions stated above.
    Then, the geodesic kNN regressor of Eq. \eqref{eq:inductive}
    computed using  the distance-cutoff rule with r as in Theorem B,
    or a symmetric kNN rule with a suitable k, satisfies
    \begin{align} \label{eq:main_theorem}
        &\expect{( \hat{f} \left( \x \right) - f(\x) )^2 }
        \le
        c n^{-\frac{2}{2+d}}
        +
        c' e^{-c'' \cdot(n+m)} f_D^2.
    \end{align}
%    \begin{align*}
%        &\expect{( \hat{f} \left( \x \right) - f(\x) )^2 }
%        \le
%        \left(\frac{4 L^2}{(1-e^{-Q})^2}
%        +
%        *** \right) \\
%        n^{-\frac{2}{2+d}} \nonumber\\
%        &&\left(2e^{-Q R^d (n+m)} \right) \cdot f_D^2.   \nonumber
%    \end{align*}
    where $f_D := f_{\max} - f_{\min}$.
    The coefficients $c,c',c''$ are independent of the sample size. They depend only on the Lipschitz constant of $f$, the noise level $\sigma$,
    properties of $\M$ and $\mu$ and on the parameters $\epsilon, \delta$ in Theorem B.
\end{theorem}
\begin{proof}
    By Eq. (\ref{eq:inductive}), \(\hat f(\x)=\hat f(\x^*)$, where $\x^*$ is the nearest point to $\x$\ from $\mathcal L\cup\mathcal U$. Since
    \(
        (a+b)^2 \le 2 a^2 + 2 b^2
    \)
    \begin{align}
        &\expect{(\hat{f}(\x) - f(\x))^2}
        =
        \expect{(\hat{f}(\x^*) - f(\x))^2} \nonumber \\
        &=
        \expect{\left( ( \hat{f} \left( \x^* \right) - f(\x^*) ) + ( f(\x^*)  -  f(\x) ) \right)^2 } \nonumber \\
        &\le
        2\expect{( \hat{f}\left(\x^*\right) - f(\x^*))^2}
        +
        2\expect{\left( f\left(\x^*\right) - f(\x)\right)^2}. \nonumber 
    \end{align}
    Bounds on these two terms are given by lemmas \ref{lemma:fxstar_minus_fx_squared}
    and \ref{lemma:fhat_xstar_minus_f_xstar} below.
    In each of these lemmas the bound is composed of a term
    \( c_1 n^{-\frac{2}{2+d}} \)
    and an exponential term of the form $c_2 e^{-c_3(n+m)} f_D^2$.
    Hence, Eq. \eqref{eq:main_theorem} follows.
\end{proof}
\begin{remark}
    While the exponential term in Eq. \eqref{eq:main_theorem}
    may be huge for small sample sizes, if the number of \emph{unlabeled} samples is large enough
    then it is guaranteed to be small with respect to the first term for \emph{any} number of labeled samples $n$.
    It thus can be absorbed into the coefficient $c$ with negligible effect.
\end{remark}
\begin{remark}
    \citet[Theorem 1]{Kpotufe2011} proved  that for data sampled from an unknown manifold, even classical (supervised) kNN based on Euclidean distances achieves the minimax bound up to log factors.
    However, his result requires $O(\log n)$ labeled points in a small Euclidean ball around $\x$.
    This is different from our result that holds for any number of labeled points  $n$, and does not include log factors.
\end{remark}

We now state and prove the two lemmas used in the proof of Theorem \ref{thm:minimax_optimality}.
To this end, let $X_1, \ldots, X_{n+m} \simiid \mu$, and let $Y_i=f(X_i)+\eta_i$ be the observed responses at the first $n$ (labeled) points, where
$\eta_1,\ldots,\eta_n\simiid \mathcal N(0,\sigma^2)$.
\begin{lemma} 
        \label{lemma:fxstar_minus_fx_squared}

Let $\x \in \M$  and let $\x^*$ be its
Euclidean nearest point from  $\{X_1, \ldots, X_{n+m} \}$.
For any $L$-Lipschitz function $f$ and measure $\mu$
that satisfies $\mu(B_{\bf z}(r)) \ge Q r^d$ 
for all $ r \le R$ and ${\bf z}\in\mathcal M,$ 
    \begin{eqnarray}
        \expect{\left( f \left( \x^* \right) - f(\x)\right)^2}
        & \le &
        \frac{2 L^2}{(1-e^{-Q})^2}
        n^{-\frac{2}{2+d}} + \nonumber\\
        &&e^{-Q R^d (n+m)} \cdot f_D^2.   \nonumber
    \end{eqnarray}
\end{lemma}
\begin{proof}
    Let $E_R$ denote the event that $d_\M(\x, \x^*) \le R$.
    \begin{align} \label{eq:E_f_xstar_minus_fx_squared}
        &\expect{\left( f(\x^*) - f(\x)\right)^2} \\
        &\le
        \pr{E_R} \cdot  \expect{\left( f(\x^*) - f(\x)\right)^2|E_R}
        +
        \pr{E_R^c} f_D^2. \nonumber
    \end{align}
    Since  $\mu(B_{\x}(r)) \ge Q r^d$ for any $r \le R$,
    \begin{align*}
        \pr{E_R^c}
        &= 
        \pr{d_{\M}(\x, \x^*) > R} \\
        &\le
        (1-QR^d)^{n+m}
        \le
        e^{-Q R^d (n+m)}.
    \end{align*}
    Next, we bound the first term of \eqref{eq:E_f_xstar_minus_fx_squared}.
Since \(f\) is $L$-Lipschitz w.r.t. the manifold,
    \begin{align*}
        \expect{\left( f(\x^*) - f(\x)\right)^2|E_R}
        \le
        L^2 \expect{d^2_\M(\x^*, \x) | E_R}.
    \end{align*}
Recall that for a non-negative random variable, $\expect{Z} = \int_0^\infty \pr{Z > t}dt$. Applying this to $d_\M^2(\x^*, \x)$ 
    \begin{align*}
&\expect{d^2_\M(\x^*, \x) |E_R} \!=
        \int_0^{\text{diam}(\M)} \! \pr{d^2_{\M}(\x^*, \x)>t | E_R} dt \\
        &= \int_0^{\text{diam}(\M)} \frac{\pr{d^2_{\M}(\x^*, \x)>t \text{ and } E_R}}{\pr{E_R}} dt \\
& = \frac{1}{\pr{E_R}} \int_0^{R^2} \pr{d_\M^2(\x^*, \x) \in(t,R^2)} dt \\
& \leq  \frac{1}{\pr{E_R}}  \int_0^{R^2} \pr{  d_\M(\x^*, \x) > \sqrt{t} } dt.
        \end{align*}        
    Lemma \ref{lemma:probability_integral} in the supplementary gives the following bound,
    which is independent of $R$
    \begin{align*}
        &\int_0^{R^2} \pr{d_\M(\x^*, \x) > \sqrt{t} } dt \\
        &\le
        2 (1-e^{-Q})^{-2}(n+m)^{-\frac{2}{d}}
        \le
        2 (1-e^{-Q})^{-2} n^{-\frac{2}{2+d}}.
    \end{align*}
Combining all of the above concludes the proof.\end{proof}

\begin{lemma} \label{lemma:fhat_xstar_minus_f_xstar}
        Under the same conditions of Lemma \ref{lemma:fxstar_minus_fx_squared},
    \begin{align*}
        &\expect{(\hat{f}(\x^*) - f(\x^*))^2} \\
        &\le
        \left( 2 L^2 \big( \tfrac{1+\delta}{1-\delta} \big)^2 c_1(\M,\mu,\delta) + \sigma^2 \right)
        n^{-\frac{2}{2+d}} \\
        &+
        4 c_a e^{-c_b \mu_{\min} \cdot (n+m)} f_D^2.
    \end{align*}
    where $\delta$ is the approximation ratio of Eq. (\ref{eq:isomap_distance_bounds}).
    The coefficients \(c_{a}\) and \(c_{b}\) depend on $\delta$ and on the manifold $\M$
    and graph construction parameters (see Eq. \eqref{eq:epsilon_bound}).
\end{lemma}
\begin{proof}
    By the bias-variance decomposition and the law of total variance,
    \begin{align} \label{eq:bias_variance_decomp}
        &\expect{(\hat{f}(\x^*) - f(\x^*))^2}
        =
        \text{bias}^2 \left( \hat{f} ( \x^*) \right) + \var{\hat{f}(x^*)} \nonumber \\
        &=
        \text{bias}^2 \left( \hat{f} ( \x^*) \right)
        +
        \mathbb{E} \left[ \var{ \hat{f}(\x^*) | X_1, \ldots, X_{n+m} } \right] \nonumber \\
        &+
        \var{\mathbb{E}\left[ \hat{f}(\x^*) | X_1, \ldots, X_{n+m} \right]} .
    \end{align}
    We now bound these three terms separately.
    We start with the bias term, which we split into two parts,
    depending on the event $A$
    \begin{eqnarray*}
        \text{bias} ( \hat{f} ( \x^*) ) 
        &=&        
        \pr{A} \cdot \text{bias} ( \hat{f} ( \x^*)  | A ) +\\
        &&
        \pr{A^c} \cdot \text{bias} ( \hat{f} ( \x^*) | A^c ) \\
        &\le&
        (1-\epsilon) \cdot \text{bias} ( \hat{f} ( \x^*)  | A ) + \epsilon \cdot f_{D}.
    \end{eqnarray*}
    Therefore,
    \begin{align}
        &\text{bias}^2 ( \hat{f} ( \x^*) )
        \le
        (1-\epsilon)^2 \text{bias}^2 ( \hat{f} ( \x^*)  | A ) \nonumber \\
        &+ 2 \epsilon(1-\epsilon) \text{bias} ( \hat{f} ( \x^*)  | A ) f_D 
        + \epsilon^2  f_D^2 \nonumber\\
        &\le
        \text{bias}^2 ( \hat{f} ( \x^*)  | A ) \label{ineq:squared_bias}
        + 3 \epsilon f_D^2.
    \end{align}
    Denote by $X^{(i,n)}_G(\x^*)$ the $i$-th closest labeled point to $\x^*$ according to the graph distance.
Let $Y^{(i,n)}_G(\x^*)$ be its response and $\eta^{(i,n)}_G(\x^*) = Y^{(i,n)}_G(\x^*) - f( X^{(i,n)}_G(\x^*))$
the noise.
    Using this notation, the geodesic kNN regression estimate of Eq. \eqref{eq:gknn_regressor} is
    \begin{align} \label{eq:gknn_explicit}
        \hat{f}(\x^*)
        =
        \sum_{i=1}^k \frac{Y_G^{(i,n)}(\x^*)}{k}
        =
        \sum_{i=1}^k \frac{f ( X_G^{(i,n)}(\x^*) ) + \eta_G^{(i,n)}(\x^*)}{k}.
    \end{align}
    For any random variable $Z$, we have $\mathbb{E}^2 \left[ Z \right] \le \expect{Z^2}$.
    Applying this, we get
    a bound on $\text{bias}^2 ( \hat{f} ( \x^*)  | A )$.
    \begin{align}
        &\text{bias}^2 ( \hat{f} ( \x^*) | A ) \nonumber 
        =
        \mathbb{E}^2 \Big[ \tfrac{1}{k} \sum_{i=1}^{k}  f(X_G^{(i,n)}(\x^*)) - f(\x^*) \Big| A \Big] \nonumber \\
        &\le
        \mathbb{E} \Big[ \big( \tfrac{1}{k} \sum_{i=1}^{k} ( f(X_G^{(i,n)}(\x^*)) - f(\x^*)  )
\big)^2 \Big| A \Big] \label{eq:squared_bias_first_bound} \\
        &\le \mathbb{E} \Big[ \big( \tfrac{1}{k} \sum_{i=1}^{k} L \cdot d_{\mathcal{M}} ( X_G^{(i,n)}(\x^*),
\x^* ) \big)^2 \Big| A \Big].  \label{eq:squared_bias_second_bound}
    \end{align}
    Conditioned on $A$, Eq. \eqref{eq:dMXM_less_dMXG} in the supplementary gives the bound
    \[
        d_{\mathcal{M}} \left( X_G^{(i,n)}(\x^*), \x^* \right) 
        \le
        \tfrac{1+\delta}{1-\delta}
        d_{\mathcal{M}} \left( X_\M^{(i,n)}(\x^*), \x^* \right) 
    \]
    Randomly split the labeled samples $X_1, \ldots, X_n$ into disjoint subsets $S_1, \ldots, S_{k+1}$,
    such that $|S_1|=\ldots=|S_k| = \lfloor \tfrac{n}{k} \rfloor$
    and $S_{k+1}$ contains the remaining elements.
    Let $S_i(\x^*) := \argmin_{\x' \in S_i} d_\M(\x^*, \x')$ be the closest element to $\x^*$ in $S_i$. Clearly,
    \begin{align} \label{eq:split_S_i}
        \sum_{i=1}^{k} d_{\mathcal M}\left(X_\M^{(i,n)}(\x^*), \x^* \right) \le \sum_{i=1}^{k} d_{\mathcal
M}\left(S_i(\x^*), \x^* \right).
    \end{align}
    Inserting this into Eq. \eqref{eq:squared_bias_second_bound} and applying Jensen's inequality
    \begin{align}
        &\text{bias}^2 ( \hat{f} ( \x^*) | A )
        \le\nonumber 
        L^2 \mathbb{E} \Big[ \big( \frac{1}{k}  \sum_{i=1}^{k} \tfrac{1+\delta}{1-\delta} d_\M ( S_i(\x^*), \x^*  ) \big)^2 \big| A \Big] \\
        &\le
        L^2 \big( \tfrac{1+\delta}{1-\delta} \big)^2 \mathbb{E} \Big[ \frac{1}{k} \sum_{i=1}^{k} d_{\mathcal{M}}^{2} ( S_i(\x^*), \x^* ) \big| A \Big] \nonumber \\
        &=
        L^2 \big( \tfrac{1+\delta}{1-\delta} \big)^2 \mathbb{E} \big[ d_{\mathcal{M}}^{2} \left( S_1(\x^*), \x^* \right) \big| A \big]. \nonumber
    \end{align}
    The set $S_1$ is simply a random draw of $ \lfloor \tfrac{n}{k} \rfloor $ points.
    By Lemma \ref{lemma:manifold_dist_to_gnn} in the supplementary,    
    \(
        \expect{d_{\mathcal{M}}^{2} \left( S_1(\x^*), \x^* \right) \Big| A }
        \le
        c_1(\M,\mu,\delta) \lfloor \tfrac{n}{k} \rfloor^{-\frac{2}{d}}.
    \)
    Plugging this back into Eq. \eqref{ineq:squared_bias}, we obtain a bound on the squared bias.
    \begin{align} \label{ineq:squared_bias_bound}
        \text{bias}^2
        \le
        L^2 \left( \tfrac{1+\delta}{1-\delta} \right)^2 c_1(\M,\mu,\delta) \lfloor \tfrac{n}{k} \rfloor^{-\frac{2}{d}} + 3\epsilon f_D^2.
    \end{align}    
    We now bound the second term in Eq. \eqref{eq:bias_variance_decomp}.
    Consider the definition of $\hat{f}(\x^*)$ in Eq. \eqref{eq:gknn_explicit}.
    Conditioned on $X_1, \ldots, X_{n+m}$, the terms $f(X_G^{(i,n)}(\x^*))$ are constants.
    The noise $\eta$ has zero mean and is independent of the draw of $X_1, \ldots, X_{n+m}$. Therefore
    \begin{align} \label{eq:variance_bound}
        \var{\hat{f}(\x^*) \big|  X_1, \ldots, X_{n+m}}
        =
        \sigma^2 / k.
    \end{align}
    To bound the third term in \eqref{eq:bias_variance_decomp}, we note that for any real random variable $Z$ and any $c \in \mathbb{R}$, we have
    \(
        \var{Z} = \expect{(Z - \expect{Z})^2} \le \expect{(Z - c)^2}
    \).
    Hence, 
    \begin{align*}
        &\var{\expect{\hat{f}(\x^*) \big| X_{1 \ldots n+m}}}
        =
        \text{Var} \Big( \tfrac{1}{k}\sum_{i=1}^{k} f(X_G^{(i,n)}(\x^*)) \Big) \\
        &\le
        \mathbb{E} \Big[ \big( \tfrac{1}{k}\sum_{i=1}^{k} f(X_G^{(i,n)}(\x^*)) - f(\x^*) \big)^2 \Big].
    \end{align*}
    We split this expectation with respect to the event $A$ and apply the bound
    we computed for Eq. \eqref{eq:squared_bias_first_bound}.
    \begin{align} \label{ineq:f_X_G_minus_f}
        &\mathbb{E} \Big[ \big( \tfrac{1}{k}\sum_{i=1}^{k} f(X_G^{(i,n)}(\x^*)) - f(\x^*) \big)^2 \Big] \\
        &= \pr{A} \cdot \mathbb{E} \Big[ \big( \tfrac{1}{k}\sum_{i=1}^{k} f(X_G^{(i,n)}(\x^*)) - f(\x^*) \big)^2 \big| A \Big] \nonumber \\
        &+ \pr{A^c} \cdot f_D^2
        \le
        L^2 \left( \tfrac{1+\delta}{1-\delta} \right)^2 c_1(\M,\mu,\delta) \lfloor \tfrac{n}{k} \rfloor^{-\frac{2}{d}}
        +
        \epsilon f_D^2. \nonumber
    \end{align}
    To conclude, by inserting equations \eqref{ineq:squared_bias_bound}, \eqref{eq:variance_bound}
    and \eqref{ineq:f_X_G_minus_f} into Eq. \eqref{eq:bias_variance_decomp},
    and applying the bound on $\epsilon$ in Eq. \eqref{eq:epsilon_bound}, we obtain
    \begin{align}
        &\expect{(\hat{f}(\x^*) - f(\x^*))^2} \nonumber\\
        &\le
        2 L^2 \left( \tfrac{1+\delta}{1-\delta} \right)^2 c_1(\M,\mu,\delta) \lfloor \tfrac{n}{k} \rfloor^{-\frac{2}{d}} \nonumber\\
        &+
        4 c_a e^{-c_b (n+m)} \cdot f_D^2        
        +
        \frac{\sigma^2}{k}. \nonumber
    \end{align}
    The lemma follows by setting $k = \lceil n^{\frac{2}{2+d}} \rceil $.
\end{proof}

%\section{Efficient computation of geodesic nearest neighbors} \label{sec:fastgknn}
\section{Computation of geodesic kNN} \label{sec:fastgknn}
Theorem \ref{thm:minimax_optimality} shows that the geodesic kNN regressor outlined in Section \ref{sec:framework} is mini\-max optimal. In this section we describe how it can also be computed efficiently,
assuming that the graph has already been constructed (more on this in Section \ref{sec:graph_construction} below). Computing $\hat f(\x)$ for all points in a dataset reduces to the following algorithmic problem: 
%The problem of efficiently computing the geodesic kNN regressor and other related methods
%reduces to the following problem in graph algorithms:
Let $G=(V,E)$ be a weighted undirected graph and let $\mathcal L \subseteq V$ be a subset of labeled vertices.
How can we efficiently find the $k$ nearest labeled neighbors of every vertex in the graph?
Denote $n = |\mathcal L|,\ N = |V|$.
A simple approach to this problem is to first apply Dijkstra's algorithm from each of the labeled points,
forming an \(n\times N\) matrix of all pairwise shortest graph distances \(d_G(s,v)\), where $s\in\mathcal L$ and $v\in V$.
The $k$ nearest labeled vertices of the $j$\textsuperscript{th} vertex correspond to the $k$ smallest cells in the $j$\textsuperscript{th} column.
The runtime of this method is $O \left( n N\log N+n|E| \right)$ \citep{dasgupta2006algorithms}.

For $k=1$, where one computes the single nearest labeled vertex to every vertex in a graph,
the result is known as the graph Voronoi diagram, with the labeled vertices acting as the centers of the Voronoi cells.
A fast algorithm for this problem was developed by \citet{Erwig2000}.
Algorithm \ref{alg:graph_k_nn} that we present here is a generalization of his approach for any $k \ge 1$.
Before describing it, we briefly recall Dijkstra's shortest path algorithm:
Given a seed vertex $s \in V$, Dijkstra's algorithm keeps,
for every vertex $v \in V$, an upper bound on $d_G(s, v)$, denoted $u[v]$,
initialized to  $0$ if $v = s$ and to $+\infty$ otherwise.
At every iteration, the vertex $v_0$ with the lowest upper bound is \emph{visited}:
For every neighbor $v$ of $v_0$, if $u[v_0] + w(v_0, v) < u[v]$,
then the current upper bound $u[v]$ is lowered.
$v_0$ is never visited again.
%An important invariant of Dijkstra's algorithm is that every time the vertex $v_0$
%with the lowest upper bound is visited, it can be proved that $u[v_0] = d_G(s, v_0)$.
%The correctness of Dijkstra's algorithm follows easily from this fact.

The basic idea behind Algorithm \ref{alg:graph_k_nn} can be described as running $n$ instances
of Dijkstra's algorithm "simultaneously" from all labeled vertices.
This is combined with an early stopping rule whenever $k$ paths from different labeled vertices
have been found.

As in Dijkstra's algorithm, Algorithm \ref{alg:graph_k_nn}
uses a priority queue based on a Fibonacci heap
with the 3 standard operations: insert, pop-minimum and decrease-key.
We use decrease-or-insert as a shorthand for
decreasing the key of an element if it is stored in the queue, and otherwise inserting it.
Instead of storing vertices in the priority queue, as in Dijkstra's algorithm,
Algorithm \ref{alg:graph_k_nn} stores pairs $(seed, v)$ keyed by $dist$,
where $dist$ is the current upper bound on $d_G(seed, v)$.
In the supplementary we prove that whenever $(dist, seed, v)$ is popped from the queue, we have $dist = d_G(seed,v)$.
At every iteration, the pair $(seed,v_0)$ with the lowest upper bound is \emph{visited}:
we examine every neighbor $v$ of $v_0$ and possibly update the current upper bound of $d_G(seed, v)$
using a decrease-or-insert operation.
We keep a set $S_v$ for every vertex $v \in V$ to prevent multiple visits from the same seed.
%Note that for every vertex $v_0 \in V$, pairs of the form $(seed, v_0)$ are visited at most $k$ times.
\begin{algorithm}
    \caption{Geodesic k nearest labeled neighbors}
    \label{alg:graph_k_nn}
    \paragraph{Input:} An undirected weighted graph $G = (V,E,w)$ and a set of labeled vertices $\mathcal L \subseteq V$.\\
    {\bf Output: }
    For every $v \in V$ a list \(kNN[v]\) with the $k$ nearest labeled vertices to $v$ and their distances.

    \algsetup{indent=2em}
    \begin{algorithmic}
        \STATE $Q \gets$ PriorityQueue()
        \FOR{$v \in V$}
            \STATE kNN[$v$] $\gets$ Empty-List()
            \STATE $S_v \gets \phi$
            \IF{$v \in \mathcal L$}
                \STATE insert($Q$, ($v$, $v$),\  priority = 0)\\\ 
            \ENDIF
        \ENDFOR
        \WHILE{$Q \neq \phi$}     
            \STATE (seed, $v_0$, dist) $\gets$ pop-minimum($Q$)
            \STATE $S_{v_0} \gets S_{v_0} \ \cup\ \{$seed$\}$
            \IF{length(kNN[$v_0$]) < $k$}
                \STATE \text{\bf append}\ (dist, seed) \text{\bf to} kNN[$v_0$] 
                \FORALL{$v \in $ neighbors$(v_0)$}
                    \IF{length(kNN[$v$]) < $k$ and seed $\notin S_v$}
                        \STATE decrease-or-insert($Q$, (seed, $v$),\\\qquad priority = dist $+ w(v_0, v)$)
                    \ENDIF
                \ENDFOR
            \ENDIF
        \ENDWHILE
    \end{algorithmic}
\end{algorithm}

Independently of our work, this algorithm was recently described by \cite{Harpeled2016} for the $\mathcal L = V $ case. 
Furthermore, an optimization of the priority queues was proposed that bounds the runtime at
\begin{align} \label{eq:runningtime}
    O(k|V|\log|V| + k|E|).
\end{align}
In the supplementary, we give a detailed description of this method, which we dub Algorithm \ref{alg:graph_k_nn_faster}.
We formally prove the correctness of both algorithms, present asymptotic bounds on their running time
and perform an empirical runtime  comparison  using the indoor localization data set.
In our experiments, both Algorithm \ref{alg:graph_k_nn} and Algorithm \ref{alg:graph_k_nn_faster}
have a similar runtime, which is orders of magnitude faster than the na\"ive method
of computing geodesic nearest neighbors using Dijkstra's algorithm.
It is also significantly faster than standard methods to compute eigenvectors, 
as required by Laplacian eigenvector regression.

Both Algorithm \ref{alg:graph_k_nn} and Algorithm \ref{alg:graph_k_nn_faster} use memory bounded by the minimum of $O(n|V|)$ and $O(k|E|)$. The first bound follows from the fact that $Q$ cannot have more than $|\mathcal L \times V|$ elements.
The second holds since every vertex $v$ is visited at most $k$ times and may insert up to $\deg(v)$ neighbors into $Q$.
%that $n|V|$ is the total number of pairs $(seed,v) \in \mathcal L \times V$. The latter 

\begin{remark}
    \cite{BijralRatliffSrebro2011}  also proposed a variant of Dijkstra's algorithm. However, their
    method is an improvement of \emph{single-source} Dijkstra in the setting of a dense graph constructed from
    points in $R^D$, whereas the methods we discuss here compute paths from \emph{multiple sources} and are applicable to any graph.
\end{remark}

\subsection{Notes on the graph construction time} \label{sec:graph_construction}
For the construction of $G$, the straightforward approach is to compute the distances between all pairs of points
in time $O(D|V|^2)$.
In light of Eq. \eqref{eq:runningtime} this may take much longer than actually computing geodesic $kNN$ on the graph $G$.
One way to reduce the runnning time is to store all of the data points in a k-d tree \citep{Bentley1975}
a ball tree \citep{Omohundro1989}  or some other data structure for spatial queries and then find nearby neighbors for every point.
These data structures are suitable for constructing both distance-cutoff graphs
and symmetric-kNN graphs from low-dimensional data.
% as their space requirements are exponential in the dimension.
For high-dimensional data, several works have appeared in recent years which propose fast methods of constructing approximate kNN graphs
\citep{ZhangHuangGengLiu2013, WangShiCao2013}.
%We are not aware of any works focused on approximate $\epsilon$-graph constructions. 
The running time of these methods is $O(D|V|\log|V|)$ multiplied by some constant which is empirically small.
Combining these constructions with  fast algorithms for computing geodesic kNN 
yields a runtime of $O((k+D)|V|\log|V|)$.
This is much faster than many other semi-supervised methods, which typically involve
expensive calculations such as matrix inversion \citep{ZhuGhahramaniLafferty2003}
or eigenvector computation \citep{BelkinNiyogi2004}.

\section{Applications} \label{sec:application}
\subsection{Indoor localization  using Wi-Fi signals}
One motivation for our work is 
the problem of estimating the location of a mobile device in a closed environment
using its Wi-Fi signature as received by a wireless router. This problem
is gaining considerable interest in recent years due to its many potential
applications, such as indoor navigation inside large commercial spaces \citep{Liu2007}.
In indoor settings, the signal received by the router is a superposition
of multiple reflections of the same source, which differ
in their arrival time, direction and intensity.
This limits the use of classic outdoor positioning methods such as triangulation,
which require a direct line-of-sight between the transmitting device and
the receiver.

A common approach for tackling this problem, known as \textit{fingerprinting} in
the signal processing community,
is based on nearest-neighbor search.
First, a labeled set $\{(\x_i,y_i)\}_{i=1}^n$ is collected,
where $y_i \in \mathbb R^2$ is the location of the transmitter
and $\x_i$ is a feature vector extracted from the received signal.
The location of new instances is then
estimated via non-parametric regression methods such as k nearest neighbors.

For applications requiring high accuracy,
recording and maintaining a suitable labeled data set may be prohibitively
expensive.  
On the other hand, collecting vast amounts of unlabeled data may be done simply by  recording the Wi-Fi signals of various devices moving through the venue.
Indoor localization
is thus a natural application for semi-supervised methods.
Moreover, the space of feature vectors is parameterized
by a 2 or 3 dimensional position. Thus, we expect manifold-based methods to perform well in this task.
To test this empirically, we used two data sets of indoor localization: a simulated and a real data set.
A brief description follows. See the supplementary for details.

\begin{figure} 
    \includegraphics[width=0.45\textwidth]{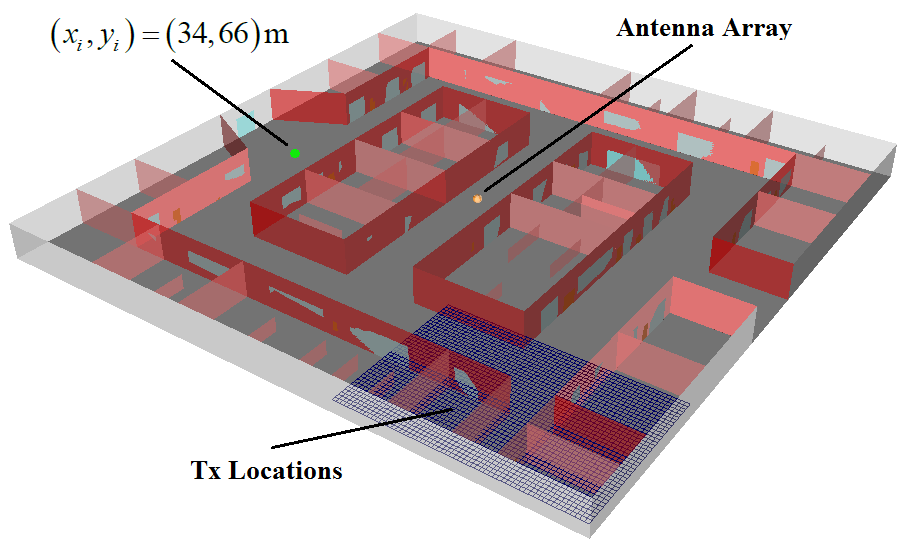}
    \caption{3D model of a $80 \times 80m \times 5m$ floor.}
    \label{fig:mall}
\end{figure}

\textbf{Simulated data:} This data consists of 802.11 Wi-Fi signals in an
artificial  $80m \times 80m$ indoor office environment  generated by \cite{KupershteinWaxCohen2013}
using a 3D radio wave propagation software, see Figure \ref{fig:mall}.

\textbf{Real data:} These are actual 802.11 signals,  recorded
by a Wi-Fi router placed roughly in the middle
of a $27m\! \times\! 33m$ office, see Figure \ref{fig:real_data_floor} of the supplementary.

The Signal Subspace Projection (SSP) of \cite{KupershteinWaxCohen2013} and \cite{JaffeWax2014}
is used as the fingerprint for localization.
It is based on the assumption that signals received from close locations
have similar properties of differential delays and directions of arrival.
In our experiments, the SSP features are $48 \times 48$ projection matrices, where the projected subspace is 10-dimensional.
We use the Frobenius norm as the distance metric and construct a symmetric-4NN graph as described in Section \ref{sec:framework}.
For more details on the datasets and SSP features, see the supplementary.

We compare our semi-supervised geodesic kNN regressor to Laplacian
eigenvector regression \citep{BelkinNiyogi2004}. As a baseline
we also applied classic kNN regression, using only the labeled samples
and optimizing over $k$.
Figure \ref{fig:localization_error} shows the median localization error
on the simulated data set as a function of the number of unlabeled locations, where the labeled points are placed on a fixed grid. 
Specifically, for the geodesic kNN regressor we used $k=7$
and exponentially decaying weights such that the weight of the $i$-th neighbor
is proportional to $1/2^i$.
Since the weights decay exponentially, the specific choice of $k$ is not important, with larger values of $k$ giving nearly identical results. 
For Laplacian eigenvector regression, we optimized over the number of
eigenvectors by taking the best outcome after repeating the experiment with
$10\%, 20\%, 30\%, 40\%$ and $50\%$ of the labeled points.

Table \ref{table:real_dataset_accuracy} shows the mean localization error on the real  data set for different densities of labeled points.
The results on both the simulated and real datasets show a clear advantage for the geodesic kNN regressor. As expected, the improvement
shown by the semi-supervised methods increases with the number of unlabeled locations.
Moreover geodesic kNN regression is much faster to compute than
the Laplacian eigenvector regressor, see Table \ref{table:runtime} in the supplementary. 

\begin{table}
    \centering
    \caption{Mean accuracy of kNN, geodesic kNN and Laplacian eigenbasis regression on the real data set}
    \label{table:real_dataset_accuracy}
    \begin{tabular}{lllll}
        \toprule
        Labeled grid & n & kNN & GNN & Laplacian\\\hline
        1.5\text{m} & 73  & 1.49\text{m} & {\bf 1.11}\text{m} & 1.36\text{m}
\\\hline % Using graph_knn=18 gknn_k=1
        2.0\text{m} & 48   & 2.27\text{m}    & {\bf 1.49}\text{m} & 1.65\text{m}
\\\hline % GKNN uses graph_knn=12 gknn_k=1 Laplacian uses graph_knn=16 n_eigenvectors=23
        %2.5 \text{meters} & 25   & 4.01 \text{meters}    & 3.14 \text{meters}
%\\\hline % Using graph_knn=9 gknn_k=1
        3\text{m}   & 23   & 3.41\text{m}    & {\bf 2.41}\text{m} & 2.79\text{m}
\\\hline % GKNN graph_knn=15 gknn_k=1 Laplacian uses graph_knn=19 n_eigenvectors=12
       \bottomrule
    \end{tabular}
\end{table}

\begin{figure}
        \includegraphics[width=\linewidth]{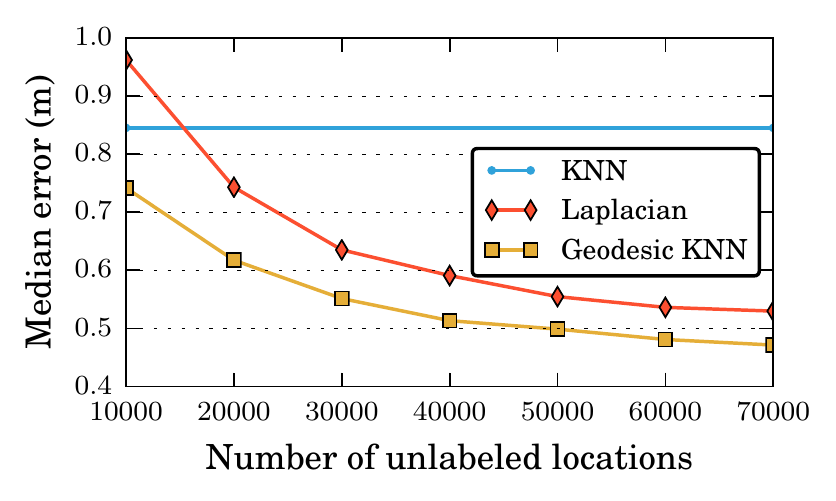}
        \includegraphics[width=\linewidth]{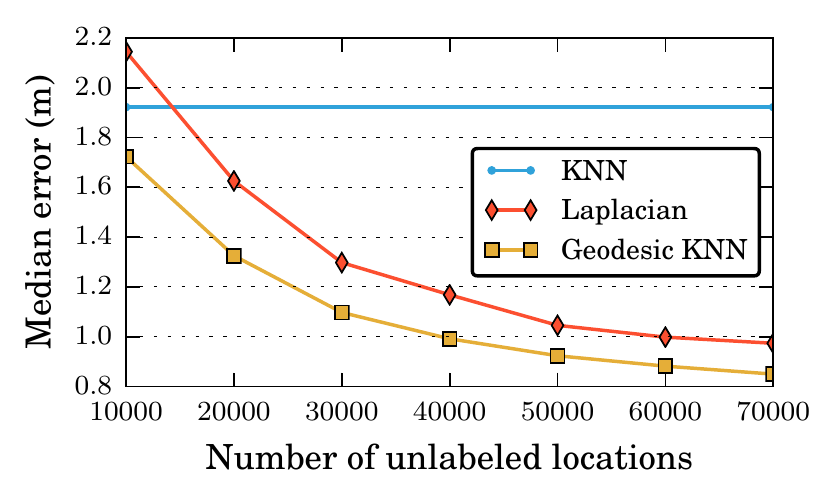}
        \caption{Median localization error vs. number of unlabeled points.
        Top: $1600$ labeled points placed on a regular grid with a side length of $2m$. Bottom: $400$ labeled points on a $4m$ grid.}
        \label{fig:localization_error}
\end{figure}
%For each algorithm, we optimized its parameters as follows:

\subsection{Facial pose estimation}

We illustrate the performance of geodesic kNN on another regression problem, using the {\bf
faces} data set
where the predicted value is the left-right angle of a face image.\footnote{\url{http://isomap.stanford.edu/datasets.html}} This data set contains 
$698$ greyscale images  of a single face rendered at different angles and
lighting.
The instance space is the set of all $64 \times 64$ images whereas the intrinsic
manifold dimension is $3$.
For our benchmark, we computed the $\ell_1$ distance between all pairs of
images
and constructed a symmetric $4$-NN graph. For the geodesic kNN algorithm,
the edge
weights were set to the $\ell_1$ distances and $k$ was set to 1.
For  Laplacian eigenvector regression we used binary weights and set the
number
of eigenvectors to 20\% of the number of labeled points.
This is a common rule-of-thumb, and gave good results over the whole range.
\begin{figure}
    \centering
    \includegraphics[width=\linewidth]{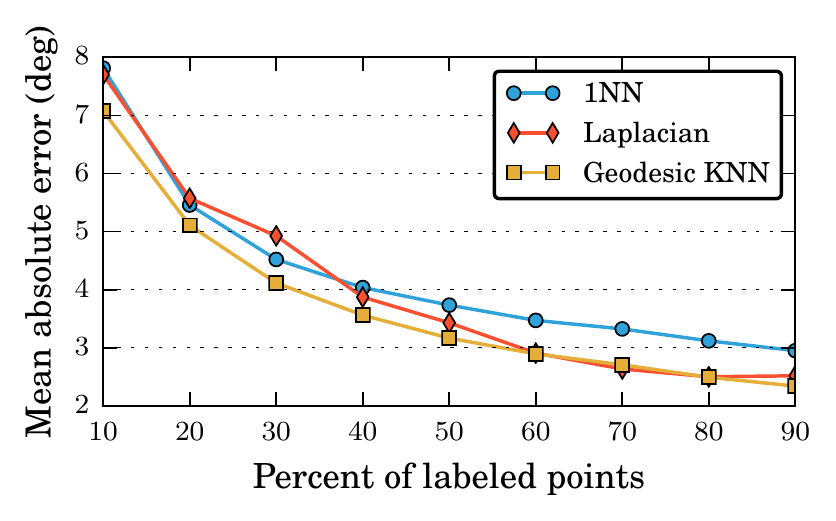}
    \includegraphics[width=0.9\linewidth]{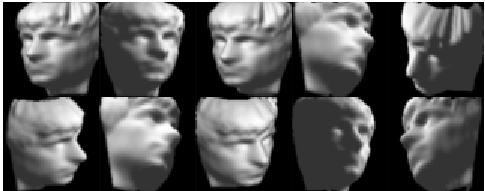}
    \caption{Top: mean prediction error for the left-right angle of the face.
    Bottom: sample images from the {\bf faces} data set, showing
different poses and lighting. }
    \label{fig:faces}
\end{figure}
Figure \ref{fig:faces} shows that geodesic kNN performs uniformly better than the nearest neighbor regressor and also outperforms the semi-supervised Laplacian regressor.

\subsubsection*{Acknowledgments}
We would like to thank
Mati Wax, Jonathan Rosenblatt, Amit Gruber, Roee David and Jonathan Bauch
for interesting discussions about this work and to Evgeny Kupershtein for providing data sets.

\pagebreak
\begin{center}
    \textbf{\Large Supplementary Material}
\end{center}

\appendix

\numberwithin{equation}{section}
\numberwithin{figure}{section}
\numberwithin{lemma}{section}
\numberwithin{theorem}{section}

\section{Auxiliary lemmas}
Consider a graph $G$ constructed from $\{ X_1, \ldots, X_{N} \}$ as described in Section \ref{sec:framework} of the main text.
For any $\x\in\{X_1,\ldots,X_N\}$, denote by $X_G^{(\ell)}(\x)$ and 
$X_{\mathcal M}^{(\ell)}(\x)$ its closest point from $X_1, \ldots X_\ell$ points, according to either the graph or the manifold distance, 
\begin{align*}
    X_G^{(\ell)}(\x) &:= \argmin_{X_j \in \{X_1, \ldots, X_\ell\}} d_G(X_i, X_j),
    \\
    X_{\mathcal M}^{(\ell)}(\x) &:= \argmin_{X_j \in \{X_1, \ldots, X_\ell\}} d_{\M}(X_i, X_j).
\end{align*}
\begin{lemma} \label{lemma:probability_integral}
    Let $X_1, \ldots, X_{\ell} \simiid \mu$ and let $\x \in \M$.
    Assume that  $\mu(B_{\x}(r)) \ge Q r^d$ for all $r \le R$.
    Then for any $\lambda > 0$, 
    \begin{align*}
        \int_{0}^{\lambda^2 R^2} \pr{d_{\M}\left(\x, X_{\M}^{(\ell)}(\x)\right) > \frac{\sqrt{r}}{\lambda}} dr
        \le
        \frac{2 \lambda^2 \ell^{-\frac{2}{d}}}{(1-e^{-Q})^2}.
    \end{align*}
\end{lemma}
\begin{proof}
    Substituting $t = \sqrt{r}/\lambda$, the  integral becomes
    \begin{align}
        \int_{0}^{R} \pr{d_{\M}\left(\x, X_{\M}^{(\ell)}(\x)\right)
>t} 2 \lambda^2  t dt.
                                \label{eq:d_M_R_int}
    \end{align}
    Now, $d_{\M}(\x, X_{\M}^{(\ell)}(\x)) > t$ if and only if
    all  $\ell$ samples fall outside the manifold ball $B_{\x}(t)$.
    Hence,
    \begin{align*}
        \pr{d_{\M}(\x, X_{\M}^{(\ell)}(\x)) > t}
        =
        (1-\mu(B_{\x}(t)))^\ell .
    \end{align*}
    Since $\mu(B_{\x}(t)) \ge Q t^d$ 
    for all $t \le R$,
    \begin{align} \label{ineq:large_distance_to_nn}
        \pr{d_{\M}(\x, X_{\M}^{(\ell)}(\x)) > t}
        \le
        (1-Qt^d)^\ell
        \le
        e^{-Q t^d \ell}
    \end{align}
We  insert (\ref{ineq:large_distance_to_nn}) into (\ref{eq:d_M_R_int}), extend the domain of integration to $(0,\infty)$ and split it into intervals of length $\ell^{-\frac{1}{d}}$.
    Each integral is bounded separately using the fact that t is  increasing
    and $e^{-Q t^d \ell}$ is decreasing.
        \begin{align*}
        &\int_{0}^{R}
        t \cdot e^{- Qt^d\ell} dt
        <
        \sum_{k=0}^{\infty}
        \int_{k \ell^{-\frac1d}}^{(k+1)\ell^{-\frac1d}}
        t \cdot e^{- Qt^d\ell} dt
        \\
        &\le
        \sum_{k=0}^{\infty}
        \ell^{-\frac2d}(k+1) \cdot e^{- k^d Q}
        \le
        \ell^{-\frac2d}
        \sum_{k=0}^{\infty}
        (k+1)  \cdot e^{-k Q}
        \\
        &=
        \ell^{-\frac2d}
        \left(
            \sum_{k=0}^{\infty}
            e^{-kQ}
        \right)^2
        =
        \frac{\ell^{-\frac2d}}{(1-e^{-Q})^2}.
    \end{align*}
Inserting this bound into (\ref{eq:d_M_R_int}) concludes the proof. 
\end{proof}

\begin{lemma} \label{lemma:manifold_dist_to_gnn}
    Consider a fixed point $\x \in \M$ and let $\x^* = \argmin_{\x' \in \{X_1, \ldots, X_{n+m} \}} \|\x-\x'\|$
    be its Euclidean nearest neighbor from the sample. We denote by $A$ the event that all
    the inequalities in Eq. \eqref{eq:isomap_distance_bounds} hold.
    Conditioned on this event, 
    \begin{align}
        &\expect{d_{\M}^2(\x^*, X_G^{(n)}(\x^*))\big|A} \le c_1(\M,\mu,\delta) n^{-\frac{2}{d}} \label{ineq:expected_squared_d_M}
    \end{align}
\end{lemma} 
\begin{proof}
    Assume that $A$ holds.
    By applying Eq. \eqref{eq:isomap_distance_bounds} twice, we obtain an upper bound on
    the manifold distance to the closest labeled point in the graph.
    \begin{align} 
        &d_\M(\x^*, X_\M^{(n)}(\x^*))
        \ge \frac{d_G(\x^*, X_\M^{(n)}(\x^*))}{1+\delta} \label{eq:dMXM_less_dMXG} \\
        &\ge
        \frac{d_G(\x^*, X_G^{(n)}(\x^*))}{1+\delta} 
        \ge
        \frac{1-\delta}{1+\delta} d_\M(\x^*, X_G^{(n)}(\x^*)) \nonumber
    \end{align}  
    Now we apply the well-known equality for a non-negative
    random variable $\expect{X} = \int_0^\infty \pr{X > r}dr$,
    \begin{align*} 
        &\expect{d_{\M}^{2}(\x^*, X_G^{(n)}(\x^*)) \big| A}  \\
        &=
        \int \pr{d_{\M}^{2}(\x^*, X_G^{(n)}(\x^*)) 
> r \big| A} dr  \\
        &=
        \int_0^{\text{diam}(\M)^2} \pr{d_{\M}\left(\x^*, X_G^{(n)}(\x^*)\right) >
\sqrt{r} \big| A} dr  \\
        &\le
        \int_0^{\text{diam}(\M)^2} \pr{d_{\M}\left(\x^*, X_{\M}^{(n)}(\x^*)\right)
> \frac{1-\delta}{1+\delta}\sqrt{r} \bigg| A} dr  
        \\
        &=
        \int_0^{\left(\frac{1+\delta}{1-\delta}\right)^2 R^2} \pr{\ldots} dr
        +
        \int_{\left(\frac{1+\delta}{1-\delta}\right)^2 R^2}^{\text{diam}(\M)^2} \pr{\ldots} dr. 
    \end{align*}
    Lemma \ref{lemma:probability_integral} with \(\lambda=(1+\delta)/(1-\delta)\) gives a bound on the first integral.
    For the second integral, the probability inside it is monotonically decreasing, and by Eq. \eqref{ineq:large_distance_to_nn} it is in particular smaller than $e^{-Q R^d n}$. Hence, 
    \begin{align*}
        &\expect{d_{\M}^{2}(\x^*, X_G^{(n)}(\x^*)) \big| A} \\
        &\le
        \frac{2 (1+\delta)^2 }{(1-\delta)^2(1-e^{-Q})^2}
        n^{-\frac{2}{d}}
        +
        \text{diam}(\M)^2
        \cdot
        e^{-Q R^d n}.
    \end{align*}
    As a function of $n$, the second summand is negligible with respect to the first.
    Hence Eq. \eqref{ineq:expected_squared_d_M} follows.
\end{proof}

\section{Fast computation of transductive geodesic nearest neighbors}

In this section we provide more details regarding the efficient computation
of geodesic kNN for all vertices in a graph.
We begin by proving the correctness of Algorithm \ref{alg:graph_k_nn}
and analyzing its runtime.
Then, in Section \ref{sec:algorithm2} we present Algorithm \ref{alg:graph_k_nn_faster} by \cite{Harpeled2016}
which has a tighter bound on the asymptotic running time.
Finally, in Section \ref{sec:actual_runtime} we  test the empirical running time of these
algorithms on the simulated WiFi  data set.

\subsection{Proof of correctness}

We start with some auxiliary definitions and lemmas
regarding shortest paths.
Given an undirected and weighted graph \(G\), we denote paths in $G$ by $v_i \to v_j \to \ldots \to v_k$ or simply
$v_i \rightsquigarrow v_k$
when the context makes it clear what path we are referring to. 
The \textit{length} of a path is the sum of its edge weights, 
\[
    w(v_1 \to v_2 \to \ldots v_m) = \sum_{i = 1}^{m-1} w(v_i, v_{i+1}).
\]
The length of the shortest (geodesic) path between two vertices $v, v' \in V$ is denoted by $d_G(v, v')$.
A path $v \rightsquigarrow v'$ is called a \emph{shortest path} if $w(v \rightsquigarrow v') = d_G(v, v')$.
% \begin{lemma} \label{lemma:subpaths_of_shortest_paths_are_shortest}
% Any sub-path of a shortest path is itself a shortest path.
% \end{lemma}
% \begin{proof}
%     Let $v_0 \rightsquigarrow v_1 \rightsquigarrow v_2 \rightsquigarrow v_3$
%     be a shortest path ($v_0$ may be equal to $v_1$, which may be equal to $v_2$ etc.)
%     By contradiction, if the subpath $v_1 \rightsquigarrow v_2$ is not a shortest path,
%     then the shortest path $v_0 \rightsquigarrow v_1 \rightsquigarrow v_2 \rightsquigarrow v_3$ can be shortened.
% \end{proof}
%\begin{definition}
%    The set of $k$ nearest labeled vertices to $v \in V$ shall be denoted by NLV$(v, k)$.
%\end{definition}
The following lemma is the key to the correctness of Algorithms \ref{alg:graph_k_nn} and \ref{alg:graph_k_nn_faster}.
\begin{lemma} \label{lemma:shortest_paths_to_k_nearest_labeled_points}
    Let $v \in V$ be a vertex and let $s$ be its $j$-th nearest labeled vertex.
    If $s \rightsquigarrow u \rightsquigarrow v$ is a shortest path
    then $s \in$ NLV$(u, j)$, where NLV$(u, j)$ is the set of $j$ nearest labeled vertices to $u$.
\end{lemma}
\begin{proof}
Assume by contradiction that there are
$j$ labeled nodes $s_1, \ldots, s_j$ such that
\[
    \forall i: d_G(s_i, u) < d_G(s, u).
\]
Then
\begin{align}
    \label{eq:d_G_s_i_v}
    d_G(s_i, v) &\le d_G(s_i, u) + d_G(u, v)  && \text{(triangle inequality)} \nonumber \\
    &< d_G(s, u) + d_G(u, v)  && \text{(by assumption)} \nonumber \\
    %&= w(s \rightsquigarrow u) + w(u \rightsquigarrow v) && \text{(by Lemma \ref{lemma:subpaths_of_shortest_paths_are_shortest})}
%\nonumber \\
   % &= w(s \rightsquigarrow u \rightsquigarrow v) \nonumber \\
    &= d_G(s, v). 
\end{align}
where the last equality is derived from the assumption that $u$ is on a shorted path between $s$ and $v$.
Eq. \eqref{eq:d_G_s_i_v} implies that the vertices $s_1, \ldots, s_j$
are all closer than $s$ to $v$, which contradicts the assumption.
\end{proof}
\noindent We now continue to the main part of the proof.
\begin{lemma} \label{lemma:triplets_represent_actual_paths}
    For every triplet$(dist, seed, v_0)$ popped from $Q$ there is a path $seed \rightsquigarrow v_0$ of length $dist$.
\end{lemma}
\begin{proof}
    We prove the claim by induction on the elements inserted into $Q$.\\
    
    \textbf{Base of the induction:} The first $L$ inserts correspond to the labeled vertices $\{(0, s, s) :\ s \in \mathcal L \}$.
    For these triplets the claim holds trivially.\\
    
    \textbf{Induction step:}
    Any triplet inserted into $Q$ or updated is of the form
    $(dist + w(v, v_0), seed, v)$ where $(dist, seed, v_0)$
    was previously inserted into $Q$. Hence, by the induction hypothesis there exists a path $seed \rightsquigarrow v_0$ of
length $dist$.
    Since $v$ is a neighbor of $v_0$ with edge weight $w(v_0, v)$, there exists a path $seed \rightsquigarrow v_0 \to v$
of length
    $w(seed \rightsquigarrow v_0) +\ w(v_0 \to v) = dist +\ w(v_0, v)$.
\end{proof}
\begin{lemma} \label{lemma:popped_distances_are_nondecreasing}
    The distances popped from $Q$ in the main loop form a monotone non-decreasing sequence.
\end{lemma}
\begin{proof}
    This follows directly from the fact that $Q$ is a minimum priority queue
    and that the edge weights are non-negative,
    hence future insertions or updates will have a priority that is higher or equal to
    that of an existing element in the queue.
\end{proof}

\begin{lemma} \label{lemma:main_lemma}
    Every time a triplet $(dist, seed, v_0)$ is popped from $Q$, the following conditions hold
    \begin{enumerate}
        \item If $seed \in NLV(v_0, k)$ then $dist = d_G(seed, v_0)$.
        \item All pairs $(s,v) \in \mathcal L \times V$ that satisfy $d_G(s, v) < dist$ and $s \in NLV(v,k)$
        are in the $visited$ set.
    \end{enumerate}
\end{lemma}
\begin{proof}
    We prove both claims simultaneously by induction on the popped triplets.\\\\
    \textbf{Base of the induction:}
    The first $L$ triplets are equal to $\{(0,s,s) : s \in \mathcal L\}$.
    Part 1 holds since $d_G(s,s) = 0$. Part 2 holds because there are no paths shorter than 0. \\\\
    \textbf{Induction step:}\\
    
    \emph{Part 1.} (If $seed \in NLV(v_0, k)$ then $dist = d_G(seed, v_0)$)\\\\
    By Lemma \ref{lemma:triplets_represent_actual_paths}, $dist$ is the length of an actual path,
    so it cannot be smaller than the shortest path length $d_G(seed, v_0)$.
    Assume by contradiction that $d_G(seed, v_0) < dist$. There is some \emph{shortest} path
    $seed \rightsquigarrow v_p \to v_0$ of length $d_G(seed, v_0)$ ($v_p$ may be equal to $seed$, but not to $v_0$).
    Clearly $ d_G(seed, v_p) \le d_G(seed,v_0) < dist $
    and by Lemma 3.1, %\ref{lemma:shortest_paths_to_k_nearest_labeled_points}
    $seed \in NLV(v_p, k)$.
    Thus by Part 2 of the induction hypothesis $(seed, v_p) \in visited$.
    This implies that $(seed, v_p)$ was visited in a previous iteration.
    Note that it follows from $seed \in NLV(v_p, k)$ and from Lemma \ref{lemma:popped_distances_are_nondecreasing}
    that during the visit of $(seed, v_p)$ the condition $length(kNN[v_p]) < k$ was true.
    Therefore the command dec-or-insert$(Q, d_G(seed, v_p)+w(v_p, v_0), seed, v_0)$
    should have been called, but because $seed \rightsquigarrow v_p \to v_0$ is a shortest path,
    $d_G(seed, v_p)+w(v_p, v_0) = d_G(seed, v_0)$, leading to the conclusion that
    the triplet $(d_G(seed, v_0), seed, v_0)$ must have been inserted into $Q$ in a previous iteration,
    which leaves two options:
    \begin{enumerate}
        \item Either $(d_G(seed, v_0), seed, v_0)$ was never popped from $Q$, in that case it should have been popped from $Q$ in the current iteration instead of $(dist, seed, v_0)$. Contradiction.
        \item The triplet $(d_G(seed, v_0), seed, v_0)$ was popped from $Q$ in a previous iteration. However this implies a double visit of $(seed, v_0)$, which is impossible due to the use of the \emph{visited} set.
    \end{enumerate}
    \emph{Part 2.} (All pairs $(s,v) \in \mathcal L \times V$ that satisfy $d_G(s, v) < dist$ and $s \in NLV(v,k)$
        are contained in $visited$)

    Let $(s,v) \in \mathcal L \times V$ be a pair of vertices that satisfies $d_G(s,v) < dist$ and $s \in NLV(v,k)$.
    Assume by contradiction that $(s,v) \notin visited$.
    Let $s \rightsquigarrow v' \to v'' \rightsquigarrow v$ be a shortest path such that $(s,v') \in $ visited but $(s, v'') \notin $ visited.
    Such $v', v''$ must exist since  $(s,s) \in visited$
    and by our assumption $(s, v) \notin visited$.
    We note that $s$ may be equal to $v'$ and $v''$ may be equal to $v$.
    Since $(s,v') \in visited$, some triplet $(d,s,v')$ was popped from $Q$ in a previous iteration and
    by Part 1 of the induction hypothesis $d = d_G(s,v')$.
    By Lemma 3.1 %\ref{lemma:shortest_paths_to_k_nearest_labeled_points}
    we know that $s \in NLV(v', k)$,
    therefore during the visit of $(s, v')$ the condition $length(kNN[v']) < k$ was true.
    Following this, there must have been a call to
    decrease-key-or-insert$(Q, d_G(s, v')+w(v', v''), s, v'')$.
    Now,
    \begin{align*}
        &d_G(s, v') + w(v', v'')\\
        &= 
        d_G(s, v'') && \text{($s \rightsquigarrow v' \to v''$ is a shortest path)}\\
        %&\le
        %w(s \rightsquigarrow v) \\
        &\le d_G(s,v) && \text{($s \rightsquigarrow v''$ is a subpath)}
    \end{align*}
    Hence the pair $(s, v'')$ was previously stored in $Q$ with distance $d_G(s, v'') \le d_G(s,v) < dist$.
    Since $(s,v'') \notin$ visited, it should have been present in $Q$ with a smaller distance than $(s, v)$
    as key, thus $(d_G(s,v''), s, v'')$ should have been popped instead of $(dist, seed, v_0)$,
    contradiction.   
\end{proof}
Finally, we are ready to state the theorem that Algorithm \ref{alg:graph_k_nn}
produces correct output, namely that its output for every vertex \(v\in V\) is indeed the set of its $k$ nearest labeled points, as measured by the geodesic graph distance.

\begin{theorem} \label{thm:main}
    Let $G=(V,E,w)$ be a graph with non-negative weights.
    For every vertex $v \in V$ let $\mathcal{L}_v$ denote the set of labeled
vertices in the connected component of $v$ and let $\ell_v = \min\{k, |\mathcal{L}_v|\}$.
    Then
    \begin{enumerate}
        \item Algorithm \ref{alg:graph_k_nn} stops after a finite number of steps.
        \item Once stopped, for every $v \in V$ the output list $kNN[v]$ is of the form
        $ kNN[v] = [(d_G(s_1, v), s_1), \ldots, (d_G(s_\ell, v), s_{\ell_v})]$
        where $s_1, \ldots, s_{\ell_v}$ are the  nearest labeled vertices,
        sorted by their distance to $v$.
    \end{enumerate}
\end{theorem}
\begin{proof}
    Part 1 follows trivially from the fact that every pair $(seed, v_0)$ can be popped at most once
and part 2 follows from Lemma \ref{lemma:main_lemma} and an induction on the popped triplets $(d, s, v)$.
\end{proof}

%%%%%%%%%%%%%%%%%%%%%%%%%%%%%%%
\subsection{Transductive runtime analysis of Algorithm \ref{alg:graph_k_nn}} \label{sec:transductive_runtime}

Lemma \ref{lemma:shortest_paths_to_k_nearest_labeled_points} explains why Algorithm \ref{alg:graph_k_nn} can stop exploring
the neighbors of vertices whose $k$ nearest labeled neighbors were found.
As Theorem \ref{theorem:runtime1} shows, this can lead to dramatic runtime savings.

\begin{theorem}
        \label{theorem:runtime1}
    Given a graph $G=(V,E)$ with $n$ labeled vertices,
    the runtime of Algorithm \ref{alg:graph_k_nn} is bounded by
    $O \left( k|E| + N_p \log |V| \right)$
    where $N_p$ is the total number of pop-minimum operations,
    which satisfies $N_p \le \min\{n|V|,k|E|\}$.
\end{theorem}
\begin{proof}
        Recall that in a priority queue based on a Fibonacci heap all operations
        cost $O(1)$ amortized time except pop-minimum which costs $\log|Q|$.      
        The runtime is dominated by the cost of all pop-minimum operations,
        plus the total cost of traversing the neighbors of the examined vertices.
        The latter takes $O(k|E|)$ time.
        Denote the total number of pop-minimum operations by $N_p$.
        We derive two different bounds on \(N_{p}\).
        Every time a $(seed,v)$ pair is popped from $Q$, it is added to the $visited$ set,
    which prevents future insertions of that pair into $Q$.
        Hence, each pair $(seed, v) \in \mathcal L \times V$ may be popped at most once from $Q$,
        which implies that
        $
                N_p \le n |V|.
        $
        In addition,  $N_p$ is bounded by the number of insertions into $Q$.
        First, there are $n$ insertions during the initialization phase.
        Then, for each vertex $v_0 \in V$,
        the "if $length(kNN[v_0]) < k$" clause can hold true at most $k$         times for that vertex.
        Each time, the neighbors of $v_0$ are examined and up to $deg(v_0)$         neighbors are inserted into $Q$.
        This yields the second bound,
        $ N_p \le n + k |E| = O(k|E|).$
\end{proof}

\subsection{Saving unneeded extractions} \label{sec:algorithm2}
Algorithm \ref{alg:graph_k_nn} keeps a priority queue with pairs $(seed,v) \in \mathcal L \times V$.
However, once a vertex $v$ has been visited from $k$ different seed vertices, we no longer
need to process it further. Thus any later pop operations which involve $v$ are a waste of CPU cycles.
Conceptually, once $v$ is visited for the $k$\textsuperscript{th} time, we would like to purge the queue
of all pairs $(s,v)$, but this is expensive since every pop operation costs $\log|Q|$.
Instead we use an idea by \cite{Harpeled2016}, which we detail in Algorithm \ref{alg:graph_k_nn_faster}.
For every vertex $v \in V$ we keep a separate priority queue $Q_v$, which will be disabled
once $v$ is visited $k$ times.
A global priority queue $Q$ is maintained that keeps the lowest element of each of the local queues $Q_v$.
Now popping an element involves popping some $v$ from $Q$ and then popping $Q_v$.
This is followed by an insert of the new minimum element from $Q_v$ into $Q$.

\setcounter{algorithm}{1}
\begin{algorithm}
    \caption{Geodesic k nearest labeled neighbors with faster priority queue handling}
    \label{alg:graph_k_nn_faster}
    \paragraph{Input:} An undirected weighted graph $G = (V,E,w)$ and a set of labeled vertices $\mathcal L \subseteq V$.

{\bf Output: }
%    \paragraph{Output:}  
For every $v \in V$ a list \(kNN[v]\) with the $k$ nearest labeled vertices to $v$ and their distances.

    \algsetup{indent=2em}
    \begin{algorithmic}    
        \STATE $Q \gets $PriorityQueue()
        \FOR{$v \in V$}
            \STATE $Q_v \gets $PriorityQueue()
            \STATE kNN[$v$] $\gets$ Empty-List()
            \STATE $S_v \gets \phi$
            \IF{$v \in \mathcal L$}
                \STATE insert($Q$, $v$, priority = 0)
                \STATE insert($Q_v$, $v$, priority = 0) \\\ \\
            \ENDIF
        \ENDFOR
        \WHILE{$Q \neq \phi$}
            \STATE ($v_0$, dist) $\gets$ pop-minimum($Q$)
            \STATE (seed, dist) $\gets$ pop-minimum($Q_{v_0}$)
            \STATE $S_{v_0} \gets S_{v_0} \ \cup\ \{$seed$\}$
            \STATE \text{\bf append}\ (seed, dist) \text{\bf to} kNN[$v_0$]             
            \IF{length(kNN[$v_0$]) <  $k$ and $Q_{v_0} \neq \phi$}
                \STATE (newseed, newdist) $\gets$ minimum($Q_{v_0}$)
                \STATE insert($Q$, $v_0$, priority = newdist)
            \ENDIF
            \FORALL{$v \in $ neighbors($v_0$)}
                \IF{length(kNN[$v$]) < $k$ and seed $\notin S_v$}
                    \STATE decrease-or-insert($Q_v$, seed,\\\qquad\qquad priority = dist $+ w(v_0,v)$)
                    \STATE decrease-or-insert($Q$, $v$, \\\qquad\qquad priority = dist $+ w(v_0,v)$)
                \ENDIF
            \ENDFOR
        \ENDWHILE
    \end{algorithmic}
\end{algorithm}

\begin{theorem}
    \label{theorem:runtime2}
    Given a graph $G=(V,E)$ with $n$ labeled vertices,
    the runtime of Algorithm \ref{alg:graph_k_nn_faster} is bounded by
    $O \left( k|E| + k V\log |V| \right)$
\end{theorem}
\begin{proof}
    The runtime of Algorithm \ref{alg:graph_k_nn_faster}
    is bounded by $O \left( k|E| + N_p \log |V| \right)$.
    Let $v \in V$ be a vertex.     
    Pairs of the form $(seed, v)$ may be popped at most $k$ times.
    hence $N_p \le k|V|$.
\end{proof}
\begin{figure}
    \centering
    \includegraphics[height=0.35\textwidth, width=0.4\textwidth]{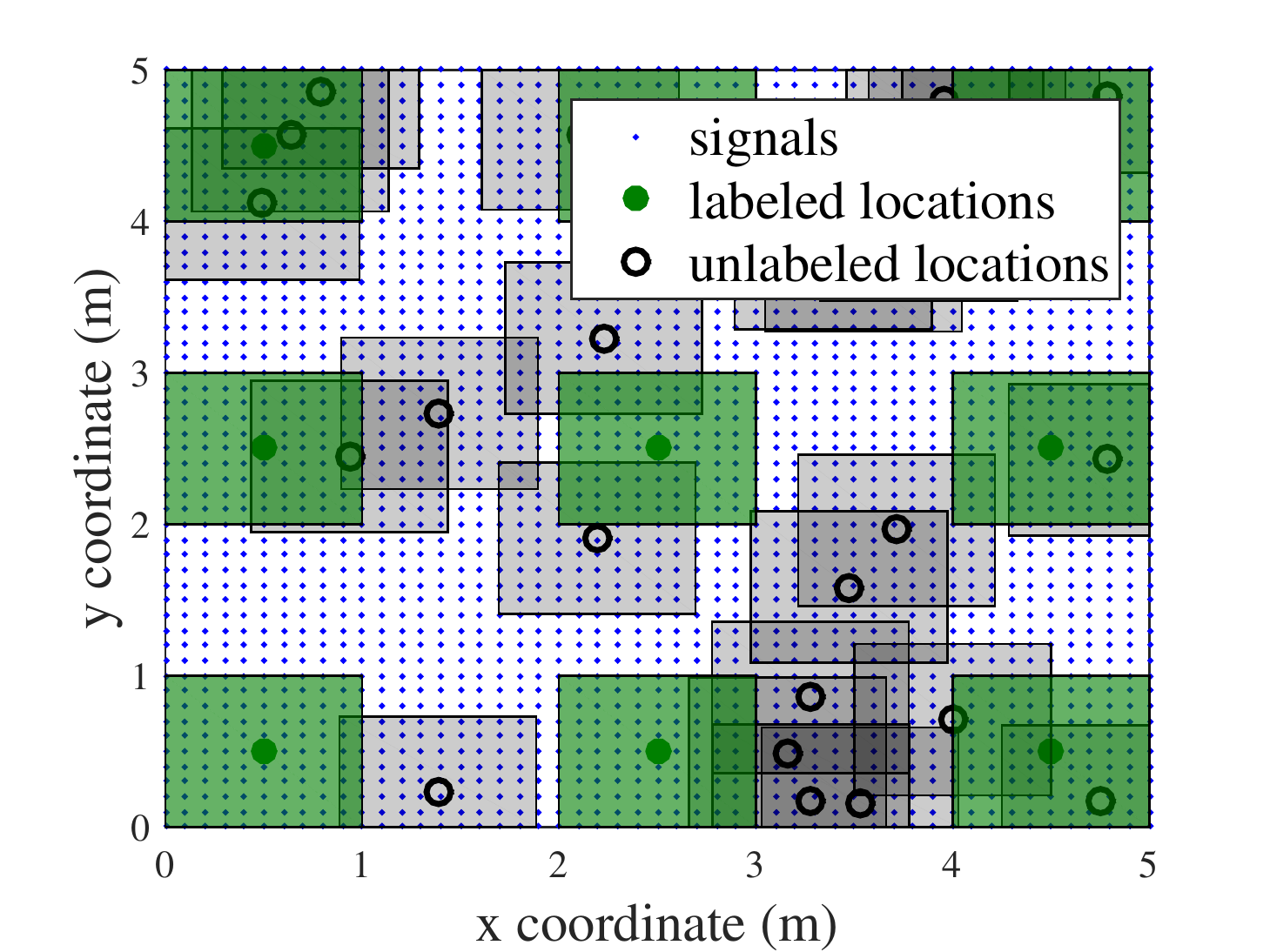}
    \caption{Schematic of the labeled and unlabeled point generation for the simulated data set. Labeled locations (green circles) were placed on a  $4m$ grid,
            whereas the unlabeled locations (grey circles) were placed at random.
            The signature of a location is computed by Eq. \eqref{eq:Rl} using all the signals in a 1m square neighborhood of the location.}
    \label{fig:signals}
\end{figure}

An immediate corollary of Theorem \ref{theorem:runtime2}
is that for a graph $G=(V,E)$ of bounded degree $d$, the runtime of Algorithm \ref{alg:graph_k_nn_faster} is
    \(
        O(k |V|\log|V|).
    \)
In contrast, the runtime of the na\"ive approach based on multiple Dijkstra runs is $O \left( n|V|\log|V| \right)$. Comparing these two formulas,
we see that major speedups are obtained in typical cases where $n \gg k$.
As we illustrate empirically in the following Section,
these speedups are very large in practice, even on graphs of moderate size.

\subsection{Empirical runtime comparison} \label{sec:actual_runtime}

We compare the runtime of Algorithms \ref{alg:graph_k_nn} and \ref{alg:graph_k_nn_faster}
to the na\"ive method of running Dijkstra's algorithm from each of the labeled points.
To make the comparison meaningful we implemented all of these algorithms in Python,
using a similar programming style and the same heap data structure.
The running times were measured for the simulated WiFi signals data set.
Figure \ref{fig:speedup} shows the relative speedup (measured in seconds) of both algorithms
compared to multiple Dijkstra runs. 
Interestingly, while Algorithm \ref{alg:graph_k_nn_faster} has better asymptotic guarantees
than Algorithm \ref{alg:graph_k_nn}, both show similar speedups on this particular data set.

\begin{figure}
        \includegraphics[width=\linewidth]{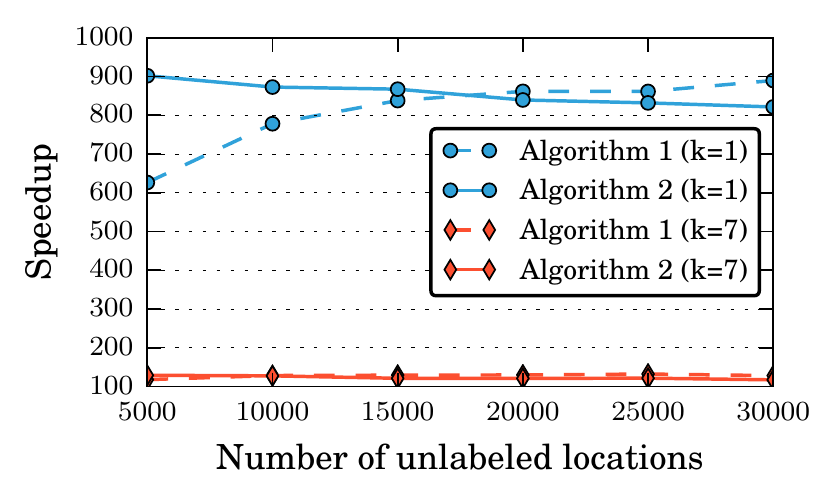}
        \includegraphics[width=\linewidth]{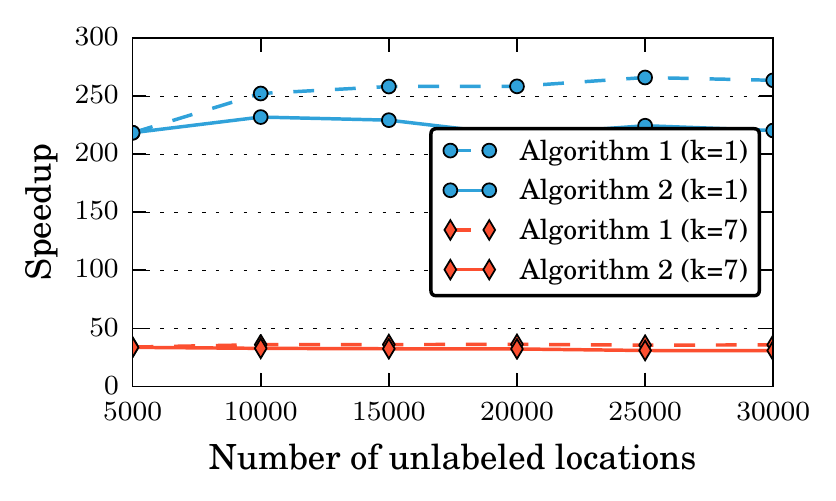}
        \caption{Relative speedup computing the geodesic 1NN and 7NN of all points in a graph
        using Algorithms \ref{alg:graph_k_nn} and \ref{alg:graph_k_nn_faster}.
        The speedup of both algorithms is compared to the na\"ive approach of running Dijkstra's algorithm from each labeled point.
        Top panel: $1600$ labeled locations are placed on a 2m square
        grid. Bottom panel: $400$ labeled locations placed on a 4m grid.}
        \label{fig:speedup}
\end{figure}
\begin{figure}
        \centering
        \includegraphics[width=\linewidth]{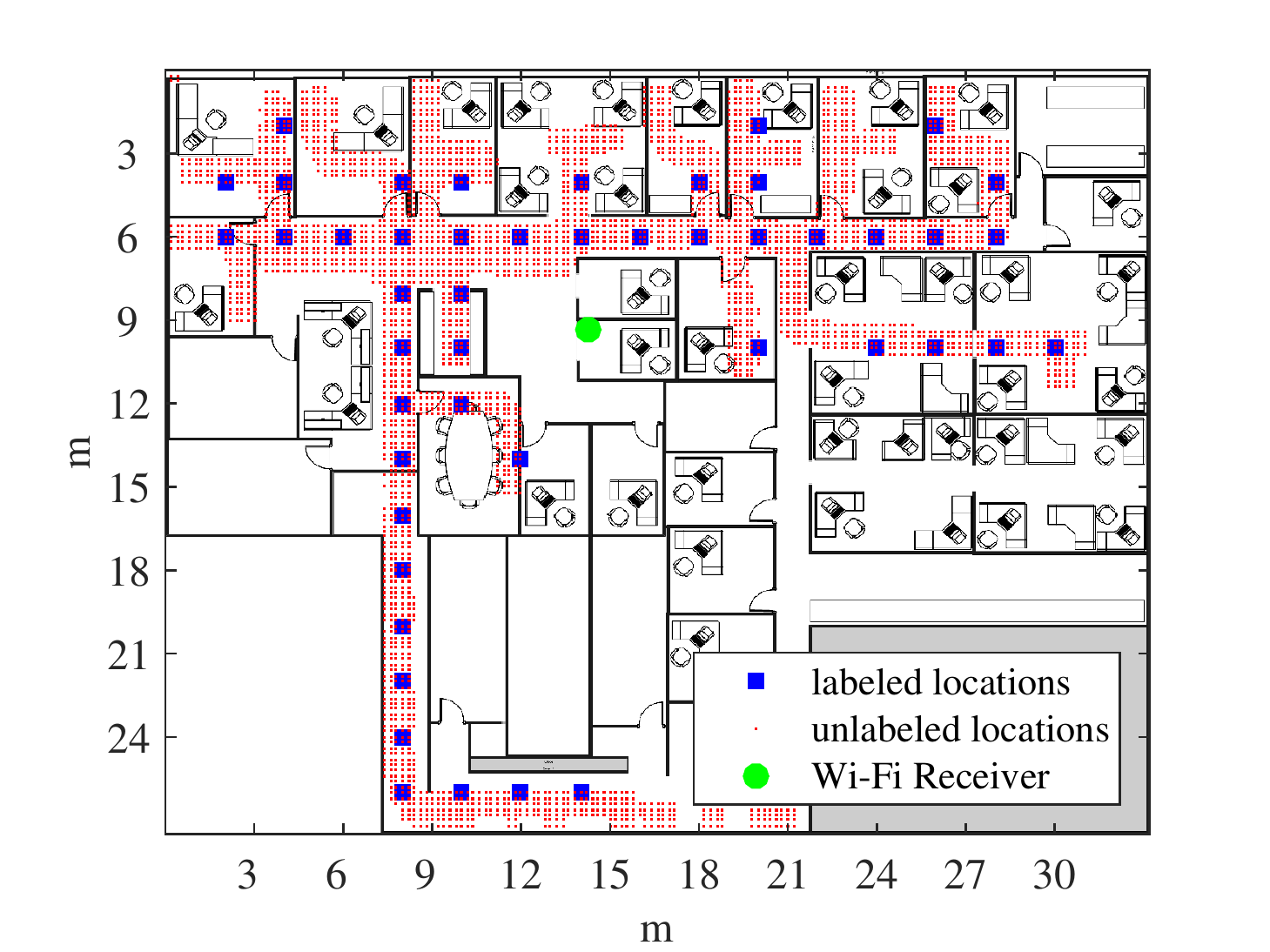}
        \caption{Schematic of the mapped areas in the real data set}
        \label{fig:real_data_floor}
\end{figure}

\begin{table}[h]
        \centering
        \caption{Runtime of Geodesic 7-NN vs. time to compute Laplacian eigenvectors}
        \label{table:runtime}
        \begin{tabular}{llll}
                \toprule
                \#unlabeled & G7NN & Laplacian & Graph build\\\hline
                1000   & 2.3\text{s} & 7.6\text{s} & 9s \\\hline
                %        2500   & 3 \text{s}   & 17 \text{seconds}  \\\hline
                %        5000   & 4.4 \text{seconds} & 60 \text{seconds}  \\\hline
                10000  & 7\text{s}   & 195\text{s} & 76s \\\hline
                %        25000  & 15 \text{seconds}  & 13 \text{minutes} \\\hline
                %        50000  & 28 \text{seconds}  & 36 \text{minutes}  \\\hline
                100000 & 56\text{s}  & 114\text{min} & 66min \\
                \bottomrule
        \end{tabular}
\end{table}

Table \ref{table:runtime} compares the runtime of the geodesic 7NN method
to the runtime of computing the top eigenvectors of the Laplacian matrix,  using the simulated indoor
localization data set with labeled
locations every $2m$.
The number of eigenvectors was chosen to be $320$, which is equal to 20\%
of the number of labeled points.
Computing the geodesic nearest neighbors by using Algorithm  \ref{alg:graph_k_nn}
is several orders of magnitude faster than computing the eigenvectors. This
is despite the fact that the eigenvector computation
is performed using the highly optimized Intel$^\circledR$  Math Kernel Library
whereas the geodesic nearest neighbor computation uses a simple Python implementation.
We expect an efficient  implementation of geodesic kNN to be at least 10
times faster.
For completeness, the third column shows the graph construction time, using
the na\"ive $O(n^2)$ algorithm.

\section{Indoor localization details}
\subsection{Dataset description}
\textbf{Simulated data:} This data consists of 802.11 Wi-Fi signals in an
artificial yet realistic environment generated by \cite{KupershteinWaxCohen2013}
using a 3D radio wave propagation software.
The environment is an $80m\times\!  80m$ floor.
In its center is a Wi-Fi router with $p=6$ antennas.
See Figure \ref{fig:mall}. At various locations $(x,y) \in \mathbb{R}^2$,
$N=8$ consecutive samples of
a Wi-Fi packet's (constant) preamble are recorded, at equally spaced time
intervals of $50 \mu s$.
The samples are stored in a complex-valued vector $s_{x,y} \in \mathbb{C}^{pN}$
which we refer to as the \emph{signal} received from location $(x,y)$.
The simulated signals were generated on a dense $0.1m$ grid covering the
entire area of the floor.

\noindent \textbf{Real data:} This data consists of actual 802.11 signals,  recorded
by a Wi-Fi router with \(p=6\) antennas placed approximately in the middle
of a $27m\! \times\! 33m$ office, see Figure \ref{fig:real_data_floor}.
The transmitter was a tablet connected to the router via Wi-Fi. The signal
vector of each location $(x,y)$ was sampled $N=8$ times from every antenna.
The transmitter locations were 
entered manually by the operator.
For both the labeled and unlabeled locations, we first generated a square grid covering the area of the office, and then kept only the locations that contained received signals.
For the labeled points, we repeated the experiments with several grid sizes ranging from $1.5m$ to $3m$.  The location of the WiFi router is marked by the green circle.

\subsection{Feature extraction and distance metric for the SSP method}

The Signal Subspace Projection method is based on the assumption that signals
originating from nearby locations are high dimensional vectors
contained in a low dimensional subspace.
Hence, signals originating from nearby locations are contained in a low dimensional
subspace. The subspace around each location is used as its signature.   
Specifically, the signature  $P_\ell$
of a location $\ell = (x,y) \in \mathbb R^2$ is computed as follows. First,
the covariance matrix of the signals in the proximity of $\ell$ is computed,
using all the signals in the dataset that are inside a $1m$ square 
around $\ell$ (see Figure \ref{fig:signals}),
\begin{align} \label{eq:Rl}
    R_{\ell} := \sum_{\ell' \in \mathbb{R}^2: \| \ell - \ell' \|_{\infty}
        < 0.5m} s_{\ell'} s_{\ell'}^*
\end{align}
where $s_{\ell'}^*$ denotes the Hermitian transpose of the column vector
$s_{\ell'} \in \mathbb{C}^{pN}$.
Next, we compute the $n_{pc}$ leading eigenvectors of $R_\ell$,
forming a matrix $V_\ell \in \mathbb{C}^{pN \times n_{pc}}$.
The SSP signature is the projection matrix onto the space spanned by these
eigenvectors, $P_{\ell} :=  V_\ell V_\ell^*$. 
In our experiments, we picked $n_{pc}=10$, though other choices in the range
$\{8, \ldots, 12\}$ gave results that are almost as good.
The distance between pairs of locations is defined as the Frobenius norm
of the difference of their projection matrices,
\(
d_{i,j} := \| P_{\bf \ell_i}-P_{\bf \ell_j}\|_F.
\)
%This simple metric, using nearest-neighbor regression and $n_{pc} = 10$
gave us results that are on par with the best methods of \cite{JaffeWax2014}.

For the simulated dataset, to mimic how a real-world semi-supervised localization
system might work, we generated the labeled locations on a regular square
grid,
whereas the unlabeled locations were randomly distributed over the entire
area of the floor (see Figure \ref{fig:signals}).
For the real dataset,
due to physical constraints, labeled and unlabeled points were not placed on a regular
grid.
Then, we created a symmetric kNN graph by connecting every point $x_i$ with
its $k_G$ 
closest neighbors, with corresponding weights given by
\(
w_{i,j} = 1+\epsilon d_{i,j}
\).
Here $\epsilon>0$ is a small constant which gives preference to paths with
smaller $d_{i,j}$.
Experimentally using the symmetric kNN graph construction with $k=4$ gave slightly better results than choosing $k \in \{3, 5, 6\}$
for both the geodesic kNN regressor and for the Laplacian eigenvector regressor.

\bibliography{geodesic-knn}
\bibliographystyle{apalike}

\end{document}